\pdfoutput=1

\documentclass[twoside,11pt]{article}
\usepackage{jmlr2e}
\usepackage{etoolbox}
\usepackage{lipsum}

\usepackage[T1]{fontenc}
\usepackage[utf8]{inputenc}

\def\withcolors{1}
\def\withnotes{1}

\ifnum\withcolors=0
  \newcommand{\new}[1]{{\color{red} {#1}}} 
\else
  \newcommand{\new}[1]{{{#1}}}

\fi

\usepackage[colorinlistoftodos,textsize=scriptsize]{todonotes}
\ifnum\withnotes=1

\else

\fi

\usepackage{amsfonts,amsmath, pbox}

\usepackage[colorlinks,citecolor=blue,bookmarks=true]{hyperref}

\usepackage{multirow}

\usepackage{dsfont} 

\usepackage{algorithmicx,algpseudocode,  algorithm}

\usepackage[shortlabels]{enumitem}
\setitemize{noitemsep,topsep=0pt,parsep=0pt,partopsep=0pt}
\setenumerate{noitemsep,topsep=0pt,parsep=0pt,partopsep=0pt}

\newcommand{\ignore}[1]{}

\newcommand{\II}{\mathbb{I}} 
\newcommand{\EE}{\mathbb{E}}

\newcommand{\expectation}[1]{\EE\left[#1\right]}
\newcommand{\expectationC}[2]{\EE_{#1}\left[#2\right]}

\def \cA     {{\cal A}}

\def \cF     {{\cal F}}

\def \cX     {{\cal X}}
\def \cY     {{\cal Y}}

\def \Paren#1{{\left({#1}\right)}}

\def\ignore#1{}

\newcommand{\bi}{\begin{itemize}}
\newcommand{\ei}{\end{itemize}}

\def\orpro{\mathop{\mathchoice
   {\vee\kern-.49em\raise.7ex\hbox{$\cdot$}\kern.4em}
   {\vee\kern-.45em\raise.63ex\hbox{$\cdot$}\kern.2em}
   {\vee\kern-.4em\raise.3ex\hbox{$\cdot$}\kern.1em}
   {\vee\kern-.35em\raise2.2ex\hbox{$\cdot$}\kern.1em}}\limits}

\def\andpro{\mathop{\mathchoice
 {\wedge\kern-.46em\lower.69ex\hbox{$\cdot$}\kern.3em}
 {\wedge\kern-.46em\lower.58ex\hbox{$\cdot$}\kern.25em}
 {\wedge\kern-.38em\lower.5ex\hbox{$\cdot$}\kern.1em}
 {\wedge\kern-.3em\lower.5ex\hbox{$\cdot$}\kern.1em}}\limits}

\def\simge{\mathrel{%
   \rlap{\raise 0.511ex \hbox{$>$}}{\lower 0.511ex \hbox{$\sim$}}}}

\def\simle{\mathrel{
   \rlap{\raise 0.511ex \hbox{$<$}}{\lower 0.511ex \hbox{$\sim$}}}}

\newcommand{\acc}{{\rm acc}}
\newcommand{\uniform}{\text{u}}
\newcommand{\match}[1]{h(#1)}
\newcommand{\bz}{\bar{z}}
\newcommand{\hbz}{\hat{\bar{z}}}
\newcommand{\hf}{\hat{f}}
\newcommand{\hz}{\hat{z}}
\newcommand{\bq}{\bar{q}}
\newcommand{\bpi}{\bar{\pi}}
\newcommand{\cAsmall}{\cA^\text{small}}
\newcommand{\cAsmallunk}{\cA^\text{small}_{\text{unknown}}}

\newcommand{\cAlarge}{\cA^\text{large}}
\newcommand{\cAlargeunk}{\cA^\text{large}_{\text{unknown}}}

\newcommand{\ham}[1]{d_{\rm ham}(#1)}
\newcommand{\uni}{\uniform_m^n}
\newcommand{\unim}{\uniform_m}
\newcommand{\bin}[2]{\text{Bin}\left(#1; #2\right)}
\newcommand{\sA}{\mathsf{A}}

\ShortHeadings{Overfitting with Hamming queries}{Acharya and Suresh}
\firstpageno{1}

\begin{document}
\title{Optimal multiclass overfitting
by sequence reconstruction from Hamming queries
}

\author{\name Jayadev Acharya\email acharya@cornell.edu \\
      \addr  Cornell University
      \AND
      \name Ananda Theertha Suresh \email theertha@google.com \\
      \addr Google Research, New York}
%


\maketitle
\begin{abstract}
A primary concern of excessive reuse of test datasets in machine learning is that it can lead to overfitting. Multiclass classification was recently shown to be more resistant to overfitting than binary classification~\citep{FeldmanFH19}. In an open problem of COLT 2019, Feldman, Frostig, and Hardt ask to characterize the dependence of the amount of overfitting bias with the number of classes $m$, the number of accuracy queries $k$, and the number of examples in the dataset $n$.  We resolve this problem and determine the  amount of overfitting possible in multi-class classification. We provide computationally efficient algorithms that achieve overfitting bias of  $\tilde{\Theta}(\max\{\sqrt{{k}/{(mn)}}, k/n\})$, matching the known upper bounds. 
\end{abstract}

\section{Introduction}
\label{sec:intro}
Training multiple machine learning models on the same training set leads to overfitting. A common method to overcome this is to divide the dataset into a training set, and a \emph{holdout} (or \emph{test}) set, where the model's accuracy on the holdout set is used as an indicator of the true generalization capability of the model~\citep{james2013introduction}. However, even the holdout dataset is used multiple times, and this leads to overfitting of models even when a holdout dataset is used~\citep{BlumH15}. In essence, training over the same dataset again and again can lead to the fake illusion of learning, all the while only making the performance over the true data distribution worse for future instances. As benchmark datasets such as MNIST, ImageNet, and others are trained by more and more machine learning algorithms in an adaptive fashion, where new models can be dependent on the previous models and their performance on the dataset, overfitting is an increasingly growing concern. 

It was recently shown that for binary classification, after $k$ interactive rounds with the test dataset it is possible to overfit the dataset by achieving an accuracy  $\Theta(\sqrt{k/n})$ larger than the true accuracy of the algorithm~\citep{DworkFHPRR15a}. To counter this,  several alternative mechanisms were proposed, such as addition of noise to the true accuracy of the predictions in each round~\citep{DworkFHPRR15b} or revealing the accuracy of prediction in a round only if it beats all previously achieved accuracies~\citep{BlumH15}. \cite{zrnic2019natural} showed improved bounds on the performance when the adaptive analyst satisfies certain constraints. This line of word broadly falls in the field of \emph{adaptive data analysis}~\citep{DworkFHPRR15b, DworkFHPRR15a, dwork2015generalization, bassily2016algorithmic}.

However, datasets such as MNIST, CIFAR10, and ImageNet have been largely immune to overfitting, even when the feedback obtained is the true accuracy \citep{recht2018cifar, recht2019imagenet, yadav2019cold}. ~\cite{recht2019imagenet} also noted that adaptivity has negligible effect on overfitting. \cite{FeldmanFH19} noted that even if these datasets are not very large, they are multiclass classification problems, where the number of possible labels is large. They considered the problem of largest overfitting possible for a multiclass classification problem, generalizing the binary results stated above, where in addition to $k$, and $n$, they  consider the role of the number of classes, henceforth denoted by $m$. In other words, they studied the following problem:

\smallskip
{\quad\emph{Given $n, k$, and $m$, by how much can adaptive algorithms overfit the test dataset?}}
%
\subsection{Prior and new results}
\label{sec:results}
\begin{table}
\centering
\begin{tabular}{ |l|l|l| }
\hline
Range of $k$ & Previous work~\citep{FeldmanFH19} & Our results \\ \hline
\multirow{2}{*}{$k=O(\frac mn)$} & $\tilde{\Omega}\left(\sqrt{\frac{k}{m^2n}}\right)\le \acc(k,n,m) - \frac{1}{m} \le \tilde{O}\left(\sqrt{\frac{k}{mn}}\right)$ & $\acc(k,n,m)= \frac1m+\tilde{\Theta}\left(\sqrt{\frac{k}{mn}}\right)$ \\
&${\rm poly}(n,k,m)$ run time & ${\rm poly}(n,k,m)$ run time\\
\hline
\multirow{2}{*}{$k={\Omega}(\frac mn)$} & $\acc(k,n,m)= \frac1m+\tilde{\Theta}\left(\frac{k}{n}\right)$ & $\acc(k,n,m)= \frac1m+\tilde{\Theta}\left(\frac{k}{n}\right)$ \\
&${\rm poly}(m,k)\cdot {\rm exp}(k)$ run time & ${\rm poly}(n,k,m)$ run time\\
\hline
\end{tabular}
\caption{Results on the overfitting bias of accuracy query based algorithms.}
\label{tab:results}
\end{table}
The question of perfect test label reconstruction dates back decades and is related to the famous game mastermind~\citep{ErdosR63, Chvatal83, doerr2016playing}. The goal here is to exactly reconstruct a sequence by making $k$ predictions, and for each prediction observing how many locations are correct. The optimal value of $k$ to perfectly reconstruct was resolved by~\cite{Chvatal83}. An efficient algorithm for $m=2$ (binary sequences) was proposed by~\cite{bshouty2009optimal}. In our setting $k$ is much smaller than needed to perfect reconstruction, and we want to understand the guarantees on how many locations can be predicted from the queries. 

\new{The general question of characterizing the overfitting bias as a function of $k,n,m$ was considered by~\cite{FeldmanFH19}. They proved an information theoretic upper bound of $1/m + \tilde{O}(
\max\{\sqrt{k/(mn)},k/n\})$ on the maximum possible accuracy. 
Note that the two terms dominate in the ranges $k=O(n/m)$ and $k=\Omega(n/m)$ respectively.}

For $k=O(n/m)$, they designed an algorithm with accuracy of $1/m+\Omega\left(\sqrt{ k/({m^2n})}\right)$. This leaves the correct relation with $m$ open, up to a quadratic factor. In other words for a given $n$ it is unclear whether the number of queries $k$ needed to achieve the same accuracy should grow linearly, quadratically or somewhere in between, as a function of the number of classes $m$. This question's resolution was the open problem in~\cite{FeldmanFH19-open}, where they also mention that from a practical viewpoint, we should give particular emphasis on computationally efficient algorithms, although even the characterization of the overfitting bias is unknown. 
Our main result resolves this question by proposing a computationally efficient algorithm that has an accuracy of $1/m+\Omega\left({\sqrt{ k/({mn})}}\right)$ for $k=O(n/m)$, matching the information-theoretic upper bound up to a logarithmic factor. The precise statement is given in Theorem~\ref{thm:small}. 

For $k=\Omega(n/m)$,~\cite{FeldmanFH19} proposed uniformly random queries over a subset of labels, and a final prediction that is not computationally efficient and achieves an accuracy of $1/m+\tilde{\Theta}\Paren{k/n}$, matching the upper bound up to logarithmic factors. For queries similar to theirs, we provide a computationally efficient final prediction that also has the optimal accuracy. We remark that all the $k$ queries of our optimal algorithms are non-adaptive, and only the final predictions depend on them. Thus adaptive  queries do not help. A summary of our results is given in Table~\ref{tab:results}. 

Finally, our proposed algorithm for small values of $k$ and the information theoretic bounds of \cite{FeldmanFH19} differ by a factor of $O(\sqrt{\log n})$. This previously known information theoretic upper bound uses minimum description length argument. For $k=1$, by a careful analysis of the geometry of the problem, we remove the $\sqrt{\log n}$ factor in Theorem~\ref{thm:info}, thus showing the optimal overfitting bias up to constant factors. It would be interesting to see if this can be extended to other values of $k$. 

\noindent\textbf{Organization.}
The rest of the paper is organized as follows. In Section~\ref{sec:problem} we give a formal problem description. In Section~\ref{sec:sequence} we consider a simplification where the test features are known, and pose it as a sequence reconstruction problem. For this sequence reconstruction problem, in Section~\ref{sec:reduction}, we show that it suffices to design algorithms where the labels are drawn from the uniform distribution on $[m]$, which allows us to only consider algorithms for this case, and in Section~\ref{sec:overview} we provide an overview of our algorithms, and in Section~\ref{sec:smallk} and Section~\ref{sec:largek}, we detail the algorithms and prove the results for $k=O(n/m)$ and $k= \Omega(n/m)$ respectively. Finally, in Section~\ref{sec:wo_features}, we solve the question in its generality where the test features can be unknown.
\section{Problem formulation}
\label{sec:problem}
Let $\cX$ denote the feature space, and $\cY=[m]:= \{1,2,\ldots, m\}$ be the set of labels. Let $S:=\{(x_1,z_1),\ldots, (x_n, z_n)\}$ be a test
set with $n$ examples with $(x_i,z_i)\in \cX\times \cY$.  A classifier $f$ is a (possibly randomized) mapping from $\cX$ to $\cY$, and the accuracy of $f$ on $S$ is
\[
\acc_S(f):= \frac1n\sum_{j=1}^n \II_{f(x_j)=z_j}.
\]
As is most common in machine learning, we consider query access to the accuracy on the dataset $S$. Each query consists of a function $f^i:\cX\to\cY$, and the accuracy oracle returns  $\acc_S(f^i)$. A \emph{$k$-query algorithm} $\cA$ makes $k$ queries $f^1, \ldots, f^k$ to $S$, and based on $f^1, \ldots, f^k$ and $\acc_S(f^1), \ldots, \acc_S(f^k)$, outputs a classifier $\hat{f}=\cA^S$. The queries are allowed to be randomized and adaptive, namely $f^i$ can depend on $f^1, \ldots, f^{i-1}$ and on $\acc_S(f^1), \ldots, \acc_S(f^{i-1})$. The accuracy of $\cA$ is
\[
\acc(\cA,S):=\expectationC{\hat{f}=\cA^S}{\acc(\hat{f})},
\]
where the expectation is over the randomization in $\cA$.  The worst case accuracy of $\cA$ is
\[
\acc(\cA):=\inf_{S}\expectationC{\hat{f}=\cA^S}{\acc_S(\hat{f})},
\]
the worse case expected accuracy over all data sets $S\in (\cX\times\cY)^n$. 
Our goal is to characterize
\[
\acc(k,n,m):=\sup_{\cA}\acc(\cA)=\sup_{\cA}\inf_{S}\expectationC{\hat{f}=\cA^S}{\acc(\hat{f})},
\]
the accuracy that can be achieved by an algorithm after making $k$ queries on any $S$. In this framework, the baseline accuracy is $1/m$, since without making any queries (when $k=0$), the best accuracy possible is $1/m$, achieved by making a uniformly random prediction for each $x\in\cX$. The \emph{overfitting bias} of an algorithm is $\acc_S(\cA)-1/m$, and we are interested in $\acc(k,n,m) - 1/m$, the maximum overfitting bias possible in the worst case. 

\noindent\textbf{A simplification.} For a test set $S$, let $S_\cX$ be the set of features $\{x_1, \ldots, x_n\}$ of the test set $S$. We first start with the variant of the problem, where the adversary has access to the test features $S_\cX$. We will remove this assumption and solve the problem in its generality in Section~\ref{sec:wo_features}. The assumption allows us to restate the overfitting problem as a sequence reconstruction problem in Section~\ref{sec:sequence}, which can be of independent interest. \new{In this case, in order to overfit on $S$, the adversary can provide its predictions on the $S_\cX$ as $f(x_1), \ldots, f(x_n)$ instead of specifying the entire function from $\cX\to[m]$. 
Hence, instead of specifying classifiers $f:\cX\to[m]$ as
queries, we specify it as length-$n$ sequences  $\bq = (q_1, q_2, \ldots, q_n)\in[m]^n$,
where $q_i = f(x_i)$. $\bq$ corresponds to our \emph{guesses} for the true labels $\bz = z_1,z_2, \ldots, z_n$ of examples in $S$.} The accuracy query oracle then returns $\frac 1n \sum_{j=1}^n\II_{q_j = z_j}$, the fraction of labels correctly predicted by the query on the test set $S$. In Section~\ref{sec:smallk}, and~\ref{sec:largek} we provide optimal overfitting algorithms in this model. Finally, in Section~\ref{sec:wo_features}, we remove this assumption and extend algorithms to the scenario when test features are unknown to the adversary.
\subsection{Sequence reconstruction from Hamming distance queries}
\label{sec:sequence}
Let $\bz = z_1, \ldots, z_n\in[m]^n$ be an unknown sequence that corresponds to the labels of examples in $S$. For a query $\bq=q_1, \ldots, q_n\in[m]^n$, an accuracy oracle returns
\[
\match{\bq , \bz} := 
\frac{1}{n} \sum^n_{i=1} \II_{q_i = z_i},
\]
the fraction of correctly predicted locations. The Hamming distance $\ham{\bq,\bz}$ between $\bq$ and $\bz$ is related to $\match{\bq,\bz}$ as
\[
\ham{\bq,\bz}= n\Paren{1-\match{\bq,\bz}}.
\]
Therefore, the query that returns the fraction of matches is equivalent to a query that returns the Hamming distance between the query and the underlying sequence. The objective is to \emph{adaptively} ask $k$ queries and then output an estimate $\hbz=\hat{z}_1, \hat{z}_2, \ldots, \hat{z}_n$ for $\bz$. The performance of the algorithm is measured by 
\[
\match{\cA, \bz} = \expectation{\match{\hat{z}_i, z_i}},
\]
where the expectation is over the algorithm's randomization. The question of perfectly reconstructing $\bz$ is well and long studied~\citep{ErdosR63, Chvatal83}, and our work resolves this problem when only partial reconstruction is possible due to limited number of queries. Similar to worst case accuracy in the previous section, the algorithms are evaluated on their worst performance
\[
\match{\cA} := \min_{\bz \in [m]^n} \match{\cA, \bz},
\]
and the goal is to find an algorithm that maximizes this worst case performance,
\[
\match{k,n,m} :=  \max_{\cA} \match{\cA}.
\]
Owing to the discussions above, we remark that
$\match{k,n,m} = \acc(k,n,m)$.
Now, under the assumption that the test set features $S_\cX$ are known to the adversary, it can provide as each query its predictions over the examples in $S_\cX$, and arbitrary predictions for $x\notin S_\cX$. Since the accuracy responses depend only on $S_\cX$, and the goal is to overfit for $S$, the question of overfitting reduces to the question of predicting a sequence under Hamming queries. 
Until Section~\ref{sec:wo_features} we consider the overfitting problem as a sequence reconstruction problem, and then in Section~\ref{sec:wo_features} generalize to the case when the test features are unknown.
\section{Reduction to average case}
\label{sec:reduction}
Instead of worst-case $\bz$, a natural question is to ask what happens if $\bz \sim p$, where $p$ is a distribution over $[m]^n$. For an algorithm
$\cA$, let
\[
\match{\cA, p} := \expectationC{\bz \sim p}{\match{\cA, \bz}}.
\]
For $p$ and any $\cA$,
\begin{equation}
\label{eq:ffobs}    
\match{\cA, p} \geq \match{\cA}.
\end{equation}
A key observation in our work is to prove that we can assume the labels to be generated from $\uni$, the uniform distribution over $[m]^n$. 
In Theorem~\ref{thm:reduction} we show that for any algorithm $\cA$, there exists an algorithm $\cA'$ such that 
\[
\match{\cA'} = \match{\cA, \uni}.
\]
In fact, we will provide an efficient construction to obtain $\cA'$ from $\cA$. 
Hence, in the rest of the paper,
we design efficient algorithms whose performance on $\uni$ matches the upper bound, and thereby proving their optimality. Theorem~\ref{thm:reduction} can also be used to show a stronger result equating the worst case and average case performance. 
\begin{corollary}
\label{cor:minmax}
For any $k,n,m$,
\[
\match{k,n,m} = \max_{\cA} \match{\cA} = \max_{\cA} \match{\cA, \uni}.
\]
\end{corollary}
\begin{proof}
Let $\cA^* = \arg \max_{\cA} \match{\cA}$. By~\eqref{eq:ffobs},
\[
\max_{\cA} \match{\cA} = \match{\cA^*} \leq \match{\cA^*, \uni} \leq  \max_{\cA} \match{\cA, \uni}.
\]
Let $\cA^*_{\uni} = \arg \max_{\cA} \match{\cA, \uni}$ be the optimal algorithm under uniform distribution. By Theorem~\ref{thm:reduction} below, there exists a 
$\cA^{*'}_{\uni}$ such that 
\[
  \max_{\cA} \match{\cA} \geq \match{\cA^{*'}_{\uni}} = \match{\cA^*_{\uni}, \uni} = \max_{\cA} \match{\cA, \uni}.
\]
The corollary follows by combining the above two equations.
\end{proof}
We now formally show the construction of $\cA'$ from $\cA$.
\begin{theorem}
\label{thm:reduction}
For any randomized and adaptive algorithm $\cA$, there exists an algorithm $\cA'$ such that
\[
\match{\cA'} = \match{\cA, \uni}.
\]
\end{theorem}
\begin{proof}
Any algorithm $\cA$ proceeds as follows. It chooses the first query $\bq^1\in[m]^n$ according to some distribution. Then for each $ i = 2,\ldots, k$, based on the previous queries $\{\bq^1, \ldots, \bq^{i-1}\}$, and accuracy responses $\{\match{\bq^{1}, \bz},\ldots, \match{\bq^{i-1}, \bz}\}$, it chooses the next (possibly randomized) query $\bq^i$. The final guess $\hbz$ is determined from all the $k$ queries and their accuracy responses.

We construct $\cA'$ from $\cA$ as follows. Let $\bpi = \pi_1, \ldots, \pi_n$ be $n$ permutations, each chosen independently and uniformly at random from $S_m$, the set of all permutations on $[m]$. For a sequence $\bz$, let $\bpi(\bz) = \pi_1(z_1), \pi_2(z_2),\ldots, \pi_n(z_n)$. Then for any $\bz$,  $\bpi(\bz)$ is distributed according to $\uni$.

Now let $\cA^{\bpi}$ be the following algorithm. If the first query of $\cA$ is $\bq^1$, the first query of $\cA^{\bpi}$ is $\bpi(\bq^1)$. Then for $ i = 2,\ldots, k$, based on the previous queries $\{\bq^j, \forall j < i\}$ and outputs $\{\match{\bpi(\bq^{j}), \bz} \forall j < i\}$, if $\cA$ queries $\bq^i$, then $\cA^{\bpi}$ queries $\bpi(\bq^i)$. Finally, if $\cA$ outputs $\hbz$, then $\cA^\pi$ outputs  $\bpi(\hbz)$. 
Now for any query $\bq$,
\[
\match{\bpi(\bq), \bz} = \frac{1}{n} \sum^n_{i=1} \II_{\pi_i(q_i) = z_i}
= \frac{1}{n} \sum^n_{i=1} \II_{q_i = \pi^{-1}_i(z_i)}
= \match{\bq, \bpi^{-1}(\bz)}.
\]
Similarly, it can be shown that for the final output
\begin{equation}
    \label{eq:red_final}
\match{\cA^{\bpi}, \bz} = \match{\bpi(\hbz), \bz} = \match{\hbz, \bpi^{-1}(\bz)}  =  \match{\cA, \bpi^{-1}(\bz)}.
\end{equation}
Therefore $\cA^{\bpi}$ achieves the same expected accuracy on $\bz$ that $\cA$ achieves on $\bpi^{-1}(\bz)$. Alternatively, $\cA^{\bpi}$ can be viewed as follows. If the first query of $\cA$ is $\bq^1$, $\cA^{\bpi}$ queries $\bq^1$ on $\bpi^{-1}(\bz)$. Then for each $ i = 2,\ldots, k$, based on the previous queries $\{\bq^j, \forall j < i\}$ and outputs $\{\match{\bq, \bpi^{-1}(\bz)}, \forall j < i\}$, if $\cA$ queries $\bq^i$, then $\cA^{\bpi}$, queries $\bq^i$ on $\bpi^{-1}(\bz)$. Finally if $\cA$ returns output $\hbz$, then $\cA^{\bpi}$ outputs $\bpi^{-1}(\hbz)$ as an estimate of $\bpi^{-1}(\bz)$. Thus by~\eqref{eq:red_final},
\[
\expectationC{\cA, \bpi}{\match{\cA^{\bpi}, \bz}} = \expectationC{\cA}{\expectationC{\bpi}{\match{\cA^{\bpi}, \bz}}} = \expectationC{\cA}{\expectationC{\bpi}{\match{\cA, \bpi^{-1}(\bz)}}} 
= \expectationC{\cA}{\match{\cA, \uni}},
\]
where the last equality uses the fact that $\bpi^{-1}(\bz)$ is distributed according to $\uni$.
Hence,
\[
\match{\cA^{\bpi}) = \min_{\bz}\expectationC{\cA, \bpi}{\match{\cA^{\bpi}, \bz}}} = \expectationC{\cA}{\match{\cA, \uni}}.
\]
Therefore, choosing $\cA'$ to be $\cA^{\bpi}$, where $\bpi$ are randomly chosen permutations proves the theorem.
\end{proof}
\section{Overview of the algorithms} 
\label{sec:overview}
By the previous section, it suffices to design algorithms assuming that the labels $\bz$ are drawn from $\uni$, namely each label is uniformly and independently distributed on $[m]$.

We first consider the case $k=O(n/m)$.~\cite{FeldmanFH19} proposed random queries, where each $q^i_j$ is independently and uniformly drawn from $[m]$. Our queries on the other hand are highly correlated across the examples. We divide the examples into essentially $k$ groups, and all the examples within a group are predicted with the same label. We will now summarize our algorithm for $k=1$, and a sketch that its overfitting bias is the optimal $\Omega(\sqrt{1/mn})$, improving from $\Omega(\sqrt{1/(m^2n)})$. The extension to larger $k$ is based on similar principles. Our single query for $k=1$ consists of predicting all the labels to be `$1$'. If the accuracy on  this query is at least $1/m$, we predict all labels as `$1$' as our final prediction, otherwise we predict all labels to be `$2$'. The number of examples with a particular label is $\text{Bin}(n;1/m)$, and for two different labels, the number of examples with two different labels are negatively associated. Using arguments about their variance, and other elementary tools, we show that this algorithm obtains a standard deviation advantage over random predictions. Here the standard deviation of the number of examples with a particular label is $\sqrt{n/m}$, which we use to prove our result. The extension to larger $k$ is similar in spirit, where we divide the examples into $k-1$ groups, and perform a similar operation over each group. The pseudo code of the algorithm is given in Figure~\ref{fig:smallk}, and a complete analysis in Section~\ref{sec:smallk}.  

When $k=\Omega(n/m)$,~\cite{FeldmanFH19} proposed an algorithm with optimal overfitting bias, which is however not computationally efficient. They choose a number $t=\tilde{\Theta}(k)$ such that it is possible to recover the labels of the first $t$ examples perfectly from the queries. They achieve this by performing uniform queries over the first $t$ examples, and constant queries over the remaining (see Figure~\ref{fig:largek}). Their guarantees are based on results from a similar problem studied in~\cite{ErdosR63, Chvatal83}, which perform a brute force search over all possible labelings of the $t$ examples, and thus are not computationally efficient. We will make a small modification to their queries for simplicity of analysis. We will also predict the last $n-t$ examples with all one's. However, for each of the first $t$ examples, we ensure that among the $k$ queries there are exactly $k/m$ of each label. Instead of reconstructing all the $t$ examples simultaneously, we predict one example's label at a time, with a success probability of at least $3/4$. We also remark that a slight modification of our algorithm can be used with the queries as proposed by~\cite{FeldmanFH19} to give an efficient optimal algorithm. 

\section{Small $k$}
\label{sec:smallk}
We show that the algorithm in Figure~\ref{fig:smallk} achieves an overfitting bias of $\Omega(\sqrt{k/mn})$ by proving the following theorem. 
\begin{theorem}
\label{thm:small}
Let $n \geq m$. For $1\le k \leq 1 + n/2m$, $\cAsmall$ in Figure~\ref{fig:smallk} satisfies 
\[
h(\cAsmall, \uni) \geq \frac{1}{m} + \frac{1}{8} \sqrt{\frac{k}{mn}}.
\]
\end{theorem}  
We prove this theorem for $k=1$, and $k>1$ separately in the next two sections. 
\begin{figure}[t]
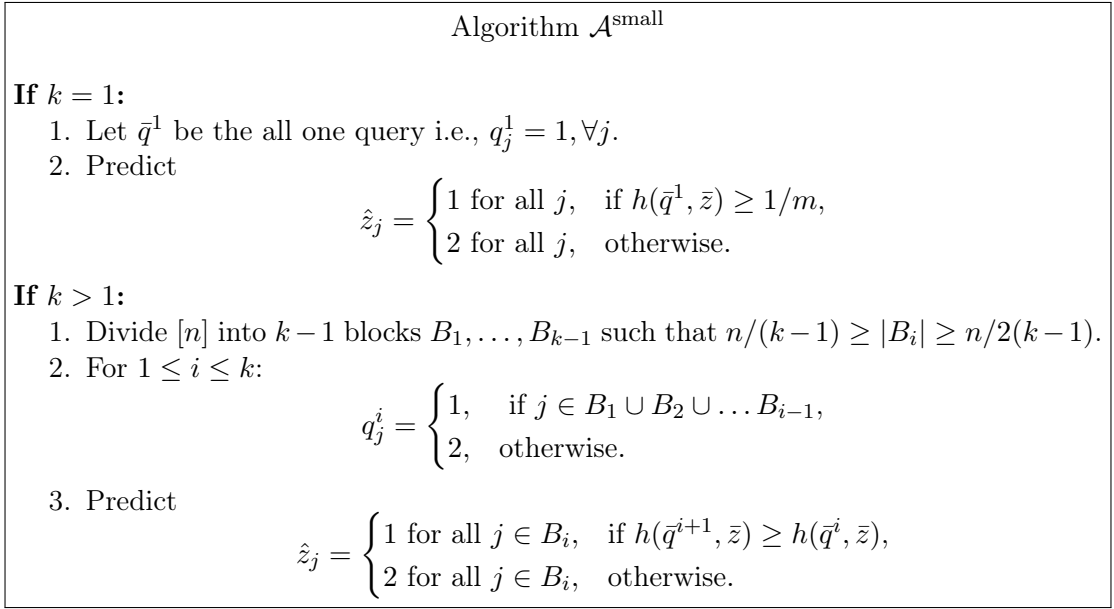

\begin{center}
\fbox{\begin{minipage}{0.95\textwidth}
\begin{center}
   Algorithm $\cAsmall$
\end{center}
\textbf{If $k = 1$:}
\begin{enumerate}
    \item Let $\bq^1$ be the all one query i.e., $q^1_j = 1, \forall j$. 
    \item Predict 
    \[
    \hat{z}_j=
    \begin{cases}
    1 \text{ for all $j$},& \text{if $\match{\bq^1, \bz} \geq 1/m$,}\\
    2 \text{ for all $j$},& \text{otherwise}.
    \end{cases}
    \]
\end{enumerate}
\textbf{If $k > 1$:}
\begin{enumerate}
\item 
Divide $[n]$ into $k-1$ blocks $B_1, \ldots, B_{k-1}$ such that $ n/(k-1) \geq |B_i| \geq n/2(k-1)$.
\item
For $1\le i \leq k$:
    \[
    q^i_j=
    \begin{cases}
    1, & \text{ if $j \in B_1 \cup B_2 \cup \ldots B_{i-1}$,}\\
    2, & \text{otherwise}.
    \end{cases}
    \]
\item Predict
\[
    \hat{z}_j=
    \begin{cases}
    1 \text{ for all $j \in B_i$},& \text{if $\match{\bq^{i+1}, \bz} \geq \match{\bq^{i}, \bz}$,}\\
    2 \text{ for all $j \in B_i$},& \text{otherwise}.
    \end{cases}
    \]
\end{enumerate}
\end{minipage}}
\end{center}
\caption{Algorithm for small values of $k$.}
\label{fig:smallk}
\end{figure}
\subsection{$k=1$} 
\noindent\textbf{The query and final prediction.}  For $\ell=1, \ldots, m$, let $N_\ell$ be the number of examples with label $\ell$. Since, the labels are uniformly distributed, $(N_1, \ldots, N_m)$ is distributed Multinomial $(n; \frac 1m, \ldots, \frac 1m)$. Our query is to predict all the labels as `$1$', namely $q^1_j=1$ for $1\le j\le n$. The accuracy observed is then $N_1/n$. If $N_1\ge n/m$, then we predict all labels as `$1$', namely $\hat{z}_j=1$ for all $j$, otherwise we output all the labels as `$2$'. The pseudocode is provided in Figure~\ref{fig:smallk}.

The number of correctly predicted labels is then given by
$N_1 \cdot \II_{N_1 \geq n/m} +  N_2 \cdot \II_{N_1 < n/m}$,
and the expected accuracy is
\[
h(\cAsmall, \uni) = \frac 1n\cdot\expectation{N_1 \cdot \II_{N_1 \geq n/m} +  N_2 \cdot \II_{N_1 < n/m}}.
\]
Hence, Theorem~\ref{thm:small} for $k=1$, follows from the following lemma. 
\begin{lemma}[Appendix~\ref{app:bin-1}]
\label{lem:bin-1}
Let $n \geq m \geq 2$. If $(N_1, \ldots, N_m)$ is distributed Multinomial $(n; \frac 1m, \ldots, \frac 1m)$,
\[
\expectation{N_1 \cdot \II_{N_1 \geq n/m} +  N_2 \cdot \II_{N_1 < n/m}} \geq  \frac{n}{m} + \frac{1}{4} \sqrt{\frac{n}{m}}.
\]
\end{lemma}
\subsection{$1< k \leq 1 + n/2m$} 
\textbf{The queries and the final prediction.} We divide the $n$ examples into $k-1$ (consecutive) blocks $B_1, \ldots, B_{k-1}$ of almost equal sizes. For $i=1, \ldots, k$, the $i$th query predicts `$1$' for all the examples in $B_1\cup\ldots\cup B_{i-1}$ and it predicts `$2$' for the remaining examples, namely $q^i_j=1$ if $j\in B_1\cup\ldots\cup B_{i-1}$, and $q^i_j=2$ otherwise. Therefore, accuracy of the $(i+1)$th query is larger than the $i$th query if and only if in $B_i$, there are more examples with label `1' than those with `2`. Our final prediction is to predict all examples in $B_i$ as `1' if there are more `1's, otherwise we predict all examples in $B_i$ as `2'. The pseudocode is given in Figure~\ref{fig:smallk}.
\begin{proof}[Proof of Theorem~\ref{thm:small} for $k>1$.]
Let $N_{i,\ell}$ be the number of examples in $B_i$ with label `$\ell$'. Then $(N_{1,\ell}, \ldots, N_{m, \ell})$ is Multinomial $(|B_i|;\frac1m, \ldots, \frac 1m)$. 
Our final predictions correctly predicts $\max\{N_{i, 1}, N_{i,2}\}$ examples in $B_i$. We use the following lemma to bound the expected overfitting bias. 
\begin{lemma}[Appendix~\ref{app:bin-max}]
\label{lem:bin-max}
Let $n' \geq m \geq 2$. If $(N_1, \ldots, N_m)$ is distributed Multinomial $(n'; \frac 1m, \ldots, \frac 1m)$,
\[
\expectation{\max\{N_{1}, N_{2}\}} \geq  \frac{n'}{m} + \frac{1}{4}\sqrt{\frac{n'}{m}}.
\]
\end{lemma}
Summing over the blocks, the expected total number of correct predictions made by our algorithm is
\begin{align}
\sum_{i=1}^{k-1}\expectation{\max\{N_{i,1}, N_{i,2}\}}
&\ge \frac{n}{m} + \frac{(k-1)}{4}\sqrt{\frac{n}{2(k-1)m}}\label{eqn:max-bound}\\
&\ge  \frac{n}{m} + \frac18\sqrt{\frac{nk}{m}},\label{eqn:last}
\end{align}
where~\eqref{eqn:max-bound} follows from Lemma~\ref{lem:bin-max} and the fact that $|B_i|\ge n/(2(k-1)) \geq m$. \eqref{eqn:last} follows from the fact that $k\ge 2$. Normalizing by $n$ proves the theorem.
\ignore{
Let $N_i(B_j)$ be the number of appearances of symbol $i$ in block $B_j$. Observe that 
\[
\match{q^{i+1}, \bz} - \match{q^i, \bz} = \frac{N_1(B_i) - N_2(B_i)}{n}.
\]
Hence, for each block, we assign the symbol in $\{1,2\}$ that appears the highest number of times. Hence, the number of matches is given by
\begin{align*}
    \match{\hbz, \bz} 
    &= \frac{1}{n} \sum^{k-1}_{j=1} \match{\hbz(B_j), \bz(B_j)} \\
    & \geq  \frac{1}{n} \sum^{k-1}_{j=1} \max\{N_1(B_j), N_2(B_j) \}\\
    & \stackrel{(a)}{\geq}  \frac{1}{n}  \sum^{k-1}_{j=1} \left(\frac{|B_j|}{m} + \sqrt{\frac{|B_j|}{4m}} \right) \\
    &  =   \frac{1}{m} +  \frac{1}{n} \sum^{k-1}_{j=1} \sqrt{\frac{|B_j|}{4m}} \\
    &  \stackrel{(b)}{\geq} \frac{1}{m} + \sqrt{\frac{k-1}{8nm}},
\end{align*}
 The proof for $k=1$ is similar and omitted. The proof follows by observing $k -1 \geq k/2$ for $k \geq 2$.
 }
\end{proof}
\section{Large $k$, $k=\Omega(n/m)$} 
\label{sec:largek}
In this section, we propose an efficient algorithm with the optimal overfitting bias for the large $k$ case. In particular, we prove the following theorem. 
\begin{theorem}
\label{thm:large}
Let $k > 9 m \log m$. Algorithm $\cAlarge$ in Figure~\ref{fig:largek} satisfies,
\[
h(\cAlarge, \uni)  \geq \frac{1}{m} + \frac{k}{36n \log m}.
\]
\end{theorem}
\noindent\textbf{Our queries.}
Our queries are a small modification to that of ~\cite{FeldmanFH19} that is slightly easier to analyze. Suppose $Q$ denote the $k\times n$ matrix whose $(i,j)$th entry is $q^i_j$. Let $t:=1 + \frac{k}{9 \log m}$.
We choose the last $n-t$ columns of $Q$ to be 1. The first $t$ columns of $Q$ are chosen independently from the following distribution: Each column is picked uniformly from all the ${k \choose \frac km, \frac km, \ldots, \frac km}$ ways such that there are exactly $k/m$ occurrences of each label $\ell\in[m]$, namely, for $1\le j\le t$, and any $\ell\in[m]$,  $|\{i:q^i_j=\ell\}|=k/m$.\footnote{We assume that $k$ is an integer multiple of $m$ for simplicity. Same results hold without the assumption.} We remark that this modification is only for simplifying the proof of optimality, and in fact we can tweak our final prediction slightly to provide an algorithm that has the optimal overfitting bias and using their queries.

\medskip
\noindent\textbf{The final prediction.}
Upon making the queries described above, we make the final prediction on  one example at a time. We predict the last $n-t$ queries as `1', namely $\hat{z}_j=1$ for $j>t$. We then show we can predict each of the first $t$ labels correctly with probability at least $3/4$. Consider the $j$th example, for $1\le j\le t$. For a label $\ell\in[m]$, consider the $k/m$ queries such that $q^i_j=\ell$, and consider the average of the accuracies returned for these queries. Our prediction for the $j$th example is the label for which this average accuracy is the largest, namely
\begin{align*}
\hat{z}_j=\arg\max_{\ell \in [m]}\sum_{i:q^i_j=\ell}\match{\bq^i, \bz}.\label{eqn:return}
\end{align*}
The queries and predictions are described in Figure~\ref{fig:largek}. We prove that our algorithm has the optimal bias up to logarithmic factors. 
\begin{figure}[t]
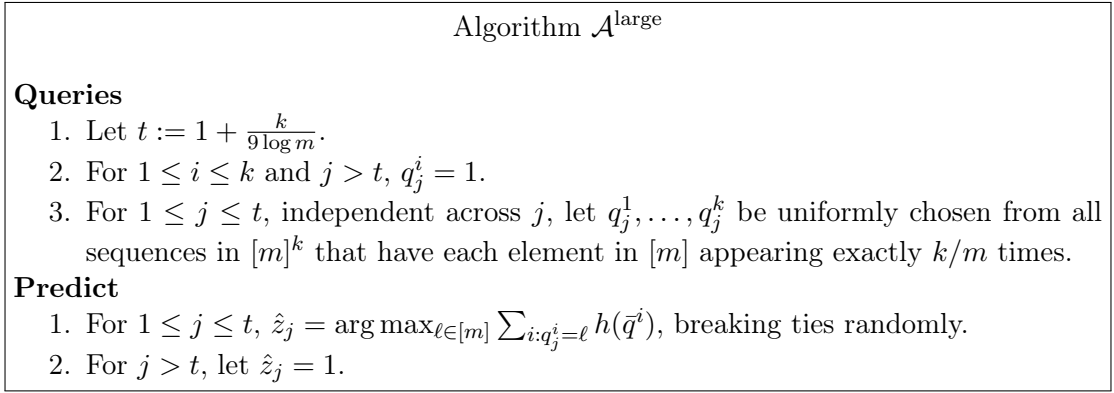

\begin{center}
\fbox{\begin{minipage}{0.95\textwidth}
\begin{center}
   Algorithm $\cAlarge$
\end{center}
\textbf{Queries}
\begin{enumerate}
\item Let $t := 1 + \frac{k }{9 \log m}$.
\item For  $1\le i \leq k$ and $j > t$, $q^i_j = 1$.
\item For $1\le j \le t$, independent across $j$, let $q^1_j, \ldots, q^k_j$ be uniformly chosen from all sequences in $[m]^k$ that have each element in $[m]$ appearing exactly $k/m$ times.
\end{enumerate}
\textbf{Predict}
\begin{enumerate}
\item For $1\le j \leq t$, $\hat{z}_j=\arg\max_{\ell \in [m]}\sum_{i:q^i_j=\ell}\match{\bq^i}$, breaking ties randomly.
\item For $j > t$, let $\hat{z}_j = 1$.
\end{enumerate}
\end{minipage}}
\end{center}
\caption{Algorithm for large values of $k$.}
\label{fig:largek}
\end{figure}
\begin{proof}[Proof of Theorem~\ref{thm:large}]
By symmetry of our queries and the reconstruction, note that the probability that $\hat{z}_j=z_j$ is the same for all $j=1, \ldots, t$. We will only consider $j=1$, and prove that $\Pr\Paren{\hat{z}_j=z_j}>3/4$. Let $\ell^*$ be the true label of the first example. Let $W$ denote the number of 1's in the last $n-t$ examples. For $\ell\in[m]$, let $A_\ell$ be the total number of correctly predicted examples by all the queries that predict the first examples as `1', \textit{i.e.,} 
\[
A_{\ell}:= n\cdot\sum_{i:q^i_1=\ell}\match{\bq^i, \bz}= \sum_{i:q^i_1=\ell}\sum_{j=1}^n \II_{q^i_j=\ell}=
\frac km\cdot \II_{\ell=\ell^*} + \frac km \cdot W + \sum_{i:q^i_1=\ell}\sum_{j=2}^t \II_{q^i_j=\ell}.
\]
 Let 
\[
M_\ell:= \sum_{i:q^i_1=\ell}\sum_{j=2}^t \II_{q^i_j=\ell},
\]
then for $\ell\ne\ell^*$
\[
A_{\ell^*}-A_{\ell} = \frac km + M_{\ell^*}-M_{\ell}.
\]
Now $q^i_j$ and $q^i_{j'}$ are independent for $j\ne j'$, namely the queries are independent across examples. Further, from basic balls and bins results for a fixed $j$, $\II_{q^i_j=\ell}$ are negatively associated across $i$. Therefore, $M_\ell$ for any $\ell$ will satisfy the Chernoff bounds: For $\varepsilon<1$
\begin{align}
\label{eqn:cher}
\Pr\Paren{|M_\ell-\EE[M_\ell]|>\varepsilon \EE[M_\ell]}\le 2\exp\Paren{-\frac{\varepsilon^2}{3}\EE[M_\ell]}.
\end{align}
Now for each $\ell$ by the linearity of expectations, 
\begin{align*}
\EE[M_\ell] = (t-1)\cdot\frac km\cdot \frac 1m.
\end{align*}
Suppose $\varepsilon$ is such that $\varepsilon\EE[M_\ell] \leq \frac k{2m}$, and $\frac{\varepsilon^2}{3}\EE[M_\ell]\geq 3\log m$, then by~\eqref{eqn:cher} and the union bound
\[
\Pr\Paren{\arg\max_{\ell \in [m]}A_\ell\ne {\ell^*}} < (m-1)\cdot \frac 2{m^3}\le \frac{1}{4},
\]
and with probability at least $3/4$, $\hbz_1=z_1$. Now,  $\varepsilon\EE[M_\ell]<\frac k{2m}$ holds for 
\[
\varepsilon\le \frac{m^2}{(t-1)k}\frac k{2m} = \frac{m}{(t-1)},
\]
and $\frac{\varepsilon^2}{3}\EE[M_\ell] \geq 3\log m$ holds for 
$
\varepsilon\ge\sqrt{9\log m\cdot \frac{m^2}{(t-1)k}}.
$
Therefore, we can find a suitable $\varepsilon$ whenever
\[
\sqrt{9\log m\cdot \frac{m^2}{(t-1)k}}\le\frac m{t-1} < 1.
\]
If we choose $t= 1+ \frac{k }{9 \log m}$ and $k > 9 m \log m$, then the condition above holds. Therefore, the expected number of correctly predicted labels is at least
\begin{align*}
\frac 34\cdot t + \frac 1m (n-t)
= \frac nm + t\cdot \Paren{\frac 34-\frac 1m}
\ge \frac nm + \frac k{36\log m},
\end{align*}
proving the result.
\end{proof}
\section{Overfitting without test features}
\label{sec:wo_features}
\new{As stated in Section~\ref{sec:problem}, the results so far assume that the adversary has knowledge of the test features. We note that the above results also hold when the test features are unknown, but the test set is indexed and is always evaluated in a particular order. In this case, the adversary can create a classifier $f: \cX \times [n] \to \cY$, that only looks at the index of the test sample and uses it to query. In particular, $f(x, i) = q_i$.}

However, in the more general setting, we may not have access to the features of the test set, and there may not be a fixed ordering of the test examples. In this case, instead of query being a length-$n$ sequence, the adversary in the $i$th query needs to provide a classifier $f^i:\cX\to\cY$. We will now generalize the algorithms in the previous sections into algorithms whose each query is a classifier over the entire feature space. \new{The guarantees for our new algorithms will be the same as those of Theorem~\ref{thm:small} and Theorem~\ref{thm:large} up to constant factors. These extensions work under a natural assumption that that all the $n$ test features in $S_\cX$ are distinct.}

Recall that $f$ is true underlying mapping from $\cX$ to $\cY$. Let $\cF$ be the set of all functions from $\cX$ to $\cY$. For a test set $S$,  $S_\cX$ is the set of features $x_1, \ldots, x_n$, i.e., the examples with their labels dropped.
With these definitions, let $\acc(\cA, S_\cX, f) := \acc(\cA, S)$. For an algorithm $\cA$ and a distribution $p$ over $\cF$,
let 
\[
\acc(\cA,  S_\cX, p) = \EE_{f \sim p} [\acc(\cA, S_\cX, f)].
\]
Similar to Theorem~\ref{thm:reduction}, we first show that uniformly random $f$ are the hardest to overfit.
 Let $\unim$ be a distribution over $\cF$ such that when $f\sim\unim$, then for each $x\in\cX$,  $f(x)$ is independently and uniformly distributed over $[m]$. Hence, as before, it suffices to consider random functions generated by $\unim$.
\begin{theorem}[Appendix~\ref{app:reduction_features}]
\label{thm:reduction_features}
For any randomized adaptive algorithm $\cA$, there exists algorithm $\cA'$ such that
\[
\acc(\cA', S) = \acc(\cA, S_\cX,  \unim). 
\]
\end{theorem}
As before,  Theorem~\ref{thm:reduction_features} can also be used to show a stronger result equating the worst case and average case performance. The proof is similar to Corollary~\ref{cor:minmax} and we omit it.
\begin{corollary}
For any $k,n,m$,
\[
\max_{\cA} \acc(\cA) = \max_{\cA} \acc(\cA, \unim).
\]
\end{corollary}

\subsection{Algorithms without test features for small $k$}
We will now provide the modifications to the previously proposed algorithms $\cAsmall$ and $\cAlarge$, which are optimal even without knowledge of the features. For $k=1$, recall that $\cAsmall$ queried using the all one query. Even when the test features are unknown, we query a function $f^1$ such that $f^1(x) =1 ,\forall x \in \cX$. For $k > 1$, recall that in $\cAsmall$, we divided the examples into $k-1$ blocks with almost equal sizes. In particular, our guarantee for $\cAsmall$ holds when the number of examples in each block is at least $2n/(k-1)$. This is possible to do when we have access to the test set features. Without knowing the features of the test set, we propose the following. Let $g:\cX\to[k-1]$ be a random mapping from $\cX$ to $[k-1]$, such that for each $x\in\cX$ $g(x)$ is independently and uniformly distributed over $[k-1]$. For $j=1, \ldots, k-1$, let 
\begin{align}
B_j:=\{x \in \cX : g(x) = j\}.\label{eqn:divideX}
\end{align}
For $k > 1$, the only modification is in step (2) of $\cAsmall$. We predict `1' for all symbols $x\in B_1\cup B_2\ldots B_{i-1}$, and `2' otherwise. 
The algorithm and the analysis is in Appendix~\ref{app:small}.

\subsection{Algorithms without test features for large $k$}
\label{sec:large-unk}
Recall that in $\cAlarge$ in Figure~\ref{fig:largek}, we made queries that ensured that each of the first $t$ examples were queried precisely $k/m$ times with each query, and the remaining $n-t$ examples are always queried with all `1's. This is not possible to do precisely without access to the features since we cannot choose a set that has exactly $t$ of the examples in $S$. We make small modifications to make it work when features are unknown. Let $\cX_t \subset \cX$ be a randomly chosen subset of $\cX$ such that each element in $x \in \cX$ is in $\cX_t$ with probability $t/n$ (this requires the knowledge of $n$). For each $x \notin \cX_t$, let $f^i(x) = 1 \forall i$. For each $x \in \cX_t$, let $f^1(x), f^2(x),\ldots, f^k(x)$ be uniformly chosen from all sequences in $[m]^k$ that have each label in $[m]$ appearing exactly $k/m$ times.  Since $\cX_t$ is chosen at random, the expected number of examples is $\cX_t$ is $t$, and by the Chernoff bound, this value concentrates around $t$, and therefore the guarantees of the algorithm still remains the same up to constant factors. The rest of the analysis is similar to that of Theorem~\ref{thm:large} and we omit it. The precise algorithm is given in Appendix~\ref{app:small}.




%
\section{Information theoretic upper bound}
\label{sec:info}
Our proposed algorithm $\cAsmall$ and the information theoretic bounds of \cite{FeldmanFH19}
differ by a factor of $O(\sqrt{\log n})$. This previously known information theoretic upper bound uses minimum description length argument. By a careful analysis that uses Corollary~\ref{cor:minmax}, we show that the $\sqrt{\log n}$ factor can be removed when $k = 1$. It would be interesting to see if this can be extended to other values of $k$.  Furthermore, for $k=1$ as the proof of Theorem~\ref{thm:info} shows $\cAsmall$ is optimal including up to the constants\footnote{We note that the results in Theorem~\ref{thm:info} and Theorem~\ref{thm:small} differ by a constant factor due to the analysis technique.}.
\begin{theorem}[Appendix~\ref{app:info}]
\label{thm:info}
For $k=1$,
\[
 \max_{\cA} \ \match{\cA} \leq \frac{1}{m} + \frac{1}{2} \sqrt{\frac{1}{n(m-1)}}.
\]
\end{theorem}

\acks{
Authors thank Vitaly Feldman, Roy Frostig, and Satyen Kale for helpful comments and suggestions. Authors thank Vitaly Feldman for suggesting methods to extend algorithms to the scenario when test features are unknown. JA is supported by NSF-CCF-1846300 (CAREER), and a Google Faculty Research Award.}
\bibliography{masterref}

\begin{thebibliography}{17}
\providecommand{\natexlab}[1]{#1}
\providecommand{\url}[1]{\texttt{#1}}
\expandafter\ifx\csname urlstyle\endcsname\relax
  \providecommand{\doi}[1]{doi: #1}\else
  \providecommand{\doi}{doi: \begingroup \urlstyle{rm}\Url}\fi

\bibitem[Bassily et~al.(2016)Bassily, Nissim, Smith, Steinke, Stemmer, and
  Ullman]{bassily2016algorithmic}
Raef Bassily, Kobbi Nissim, Adam Smith, Thomas Steinke, Uri Stemmer, and
  Jonathan Ullman.
\newblock Algorithmic stability for adaptive data analysis.
\newblock In \emph{Proceedings of the forty-eighth annual ACM symposium on
  Theory of Computing}, pages 1046--1059, 2016.

\bibitem[Berend and Kontorovich(2013)]{berend2013sharp}
Daniel Berend and Aryeh Kontorovich.
\newblock A sharp estimate of the binomial mean absolute deviation with
  applications.
\newblock \emph{Statistics \& Probability Letters}, 83\penalty0 (4):\penalty0
  1254--1259, 2013.

\bibitem[Blum and Hardt(2015)]{BlumH15}
Avrim Blum and Moritz Hardt.
\newblock The ladder: a reliable leaderboard for machine learning competitions.
\newblock In \emph{Proceedings of the 32nd International Conference on
  International Conference on Machine Learning-Volume 37}, pages 1006--1014.
  JMLR.org, 2015.

\bibitem[Bshouty(2009)]{bshouty2009optimal}
Nader~H Bshouty.
\newblock Optimal algorithms for the coin weighing problem with a spring scale.
\newblock In \emph{COLT}, volume 2009, page~82. Citeseer, 2009.

\bibitem[Chv{\'a}tal(1983)]{Chvatal83}
Vasek Chv{\'a}tal.
\newblock Mastermind.
\newblock \emph{Combinatorica}, 3\penalty0 (3-4):\penalty0 325--329, 1983.

\bibitem[Doerr et~al.(2016)Doerr, Doerr, Sp{\"o}hel, and
  Thomas]{doerr2016playing}
Benjamin Doerr, Carola Doerr, Reto Sp{\"o}hel, and Henning Thomas.
\newblock Playing mastermind with many colors.
\newblock \emph{Journal of the ACM (JACM)}, 63\penalty0 (5):\penalty0 42, 2016.

\bibitem[Dwork et~al.(2015{\natexlab{a}})Dwork, Feldman, Hardt, Pitassi,
  Reingold, and Roth]{dwork2015generalization}
Cynthia Dwork, Vitaly Feldman, Moritz Hardt, Toni Pitassi, Omer Reingold, and
  Aaron Roth.
\newblock Generalization in adaptive data analysis and holdout reuse.
\newblock In \emph{Advances in Neural Information Processing Systems}, pages
  2350--2358, 2015{\natexlab{a}}.

\bibitem[Dwork et~al.(2015{\natexlab{b}})Dwork, Feldman, Hardt, Pitassi,
  Reingold, and Roth]{DworkFHPRR15b}
Cynthia Dwork, Vitaly Feldman, Moritz Hardt, Toniann Pitassi, Omer Reingold,
  and Aaron Roth.
\newblock The reusable holdout: Preserving validity in adaptive data analysis.
\newblock \emph{Science}, 349\penalty0 (6248):\penalty0 636--638,
  2015{\natexlab{b}}.

\bibitem[Dwork et~al.(2015{\natexlab{c}})Dwork, Feldman, Hardt, Pitassi,
  Reingold, and Roth]{DworkFHPRR15a}
Cynthia Dwork, Vitaly Feldman, Moritz Hardt, Toniann Pitassi, Omer Reingold,
  and Aaron~Leon Roth.
\newblock Preserving statistical validity in adaptive data analysis.
\newblock In \emph{Proceedings of the forty-seventh annual ACM symposium on
  Theory of computing}, pages 117--126. ACM, 2015{\natexlab{c}}.

\bibitem[Erd\"os and R{\'e}nyi(1963)]{ErdosR63}
Paul Erd\"os and Alfred R{\'e}nyi.
\newblock On two problems of information theory.
\newblock \emph{Magyar Tud. Akad. Mat. Kutat\'o K\"ozl}, 8, 1963.

\bibitem[Feldman et~al.(2019{\natexlab{a}})Feldman, Frostig, and
  Hardt]{FeldmanFH19}
Vitaly Feldman, Roy Frostig, and Moritz Hardt.
\newblock The advantages of multiple classes for reducing overfitting from test
  set reuse.
\newblock In \emph{International Conference on Machine Learning},
  2019{\natexlab{a}}.

\bibitem[Feldman et~al.(2019{\natexlab{b}})Feldman, Frostig, and
  Hardt]{FeldmanFH19-open}
Vitaly Feldman, Roy Frostig, and Moritz Hardt.
\newblock Open problem: How fast can a multiclass test set be overfit?
\newblock In Alina Beygelzimer and Daniel Hsu, editors, \emph{Proceedings of
  the Thirty-Second Conference on Learning Theory}, volume~99 of
  \emph{Proceedings of Machine Learning Research}, pages 3185--3189, Phoenix,
  USA, 25--28 Jun 2019{\natexlab{b}}. PMLR.

\bibitem[James et~al.(2013)James, Witten, Hastie, and
  Tibshirani]{james2013introduction}
Gareth James, Daniela Witten, Trevor Hastie, and Robert Tibshirani.
\newblock \emph{An introduction to statistical learning}, volume 112.
\newblock Springer, 2013.

\bibitem[Recht et~al.(2018)Recht, Roelofs, Schmidt, and
  Shankar]{recht2018cifar}
Benjamin Recht, Rebecca Roelofs, Ludwig Schmidt, and Vaishaal Shankar.
\newblock Do cifar-10 classifiers generalize to cifar-10?
\newblock \emph{arXiv preprint arXiv:1806.00451}, 2018.

\bibitem[Recht et~al.(2019)Recht, Roelofs, Schmidt, and
  Shankar]{recht2019imagenet}
Benjamin Recht, Rebecca Roelofs, Ludwig Schmidt, and Vaishaal Shankar.
\newblock Do imagenet classifiers generalize to imagenet?
\newblock \emph{arXiv preprint arXiv:1902.10811}, 2019.

\bibitem[Yadav and Bottou(2019)]{yadav2019cold}
Chhavi Yadav and L{\'e}on Bottou.
\newblock Cold case: The lost mnist digits.
\newblock \emph{arXiv preprint arXiv:1905.10498}, 2019.

\bibitem[Zrnic and Hardt(2019)]{zrnic2019natural}
Tijana Zrnic and Moritz Hardt.
\newblock Natural analysts in adaptive data analysis.
\newblock In \emph{International Conference on Machine Learning}, pages
  7703--7711, 2019.

\end{thebibliography}
\newpage
\appendix
\section{Properties of the multinomial distribution}

\subsection{Proof of Lemma~\ref{lem:bin-1}}
\label{app:bin-1}
Let $N'_2$ be an independent copy of $N_2$.  Since $N_1$ and $N_2$ are negatively correlated,
\begin{align}
\expectation{N_1 \cdot \II_{N_1 \geq n/m} +  N_2 \cdot \II_{N_1 < n/m}}
& \geq \expectation{N_1 \cdot \II_{N_1 \geq n/m} +  N'_2 \cdot \II_{N_1 < n/m}}\nonumber \\
& = \expectation{N_1\cdot \II_{N_1 \geq n/m}  +  \frac{n}{m}\cdot \II_{N_1 < n/m}} \nonumber \\
& = \expectation{ \Paren{ N_1 - \frac{n}{m}}\cdot\II_{N_1 \geq n/m}} + \frac{n}{m}. \nonumber
\end{align}
Let $X$ be a random variable with $\EE[X]=a$. Since $|X-a| = (X-a)\II_{X\ge a}+ (a-X)\II_{X< a}$, 
\begin{align*}
\label{eqn:expect-abs}
\expectation{|X-a|} = 2\cdot \expectation{(X-a)\II_{X\ge a}}.
\end{align*}
Using this with $X=N_1$, and $\expectation{N_1}=n/m$ gives
\begin{align*}
\expectation{ \Paren{ N_1 - \frac{n}{m}}\cdot\II_{N_1 \geq n/m}} + \frac{n}{m} = \frac{1}{2}\expectation{ \left \lvert  N_1 - \frac{1}{m} \right \rvert} + \frac{n}{m}.
\end{align*}
\citet[Theorem 1]{berend2013sharp} showed that for $Y\sim \text{Bin}(n;p)$ with $\frac 1n\le p\le 1-\frac1n$ 
\begin{equation}
\expectation{\left\lvert Y- np\right\rvert} \ge \sqrt{\frac{np(1-p)}2}.
\label{eqn:bk13}
\end{equation}
Using this with $Y=N_1$, and $p=1/m\ge1/n$, we obtain
\begin{align*}
 \frac{1}{2}\expectation{ \left \lvert  N_1 - \frac{1}{m} \right \rvert} + \frac{n}{m} 
 \ge \frac{n}{m} + \frac12\sqrt{\frac {n}{2m}\Paren{1-\frac1m} }\ge  \frac{n}{m} + \frac{1}{4} \sqrt{\frac{n}{m}},
\end{align*}
where the final step uses $m\ge 2$. Plugging this back proves the lemma.

\subsection{Proof of Lemma~\ref{lem:bin-max}}
\label{app:bin-max}
Note that
\[
\max\{N_{1}, N_{2}\} = \frac{N_{1} + N_{2}}{2} + \frac{|N_{1} - N_{2}|}{2}.
\]
Since $N_{1}$ and $N_{2}$ are distributed $\text{Bin}(n';1/m)$, $\EE[N_{1} +
  N_{2}] = \frac{2n'}{m}$. 
Let $N'_{2}$ be an independent copy of
$N_{2}$. Since $N_{1}$ and $N_{2}$ are negatively correlated,
\begin{align}
\EE[|N_{1} - N_{2}|]  \geq \EE[|N_{1} - N'_{2}|]& \geq \EE \left[\left\lvert N_{1} - \frac{n'}{m} \right\rvert \right]\label{eqn:jensen} \\
& \geq \sqrt{\frac{\EE \left[\left( N_{1} - \frac{n'}{m} \right)^2\right]}{2}} \label{eqn:bkontorovich}\\
& = \sqrt{\frac{n'}{2m} \left( 1- \frac{1}{m}  \right)} \nonumber\\
& \geq \sqrt{\frac{n'}{4m}}\label{eqn:fin-step},
\end{align}
where~\eqref{eqn:jensen} follows from Jensen's inequality,~\eqref{eqn:bkontorovich} from~\eqref{eqn:bk13}, and~\eqref{eqn:fin-step} uses $m\ge 2$.
\section{Extensions to unknown test features}

\subsection{Proof of Theorem~\ref{thm:reduction_features}}
\label{app:reduction_features}
The proof is similar to that of Theorem~\ref{thm:reduction}.
Recall that an algorithm $\cA$ proceeds as follow. It chooses the first query $f^1$ from some distribution over $\cF$. For $ i = 2,\ldots, k$, based on the  $\{f^1, \ldots, f^{i-1}\}$, and accuracy responses $\{\acc(f^{1}, \bz),\ldots, \acc(f^{i-1}, \bz)\}$, it chooses the next, possibly randomized, $f^i$. The final guess $\hf$ is designed based on all the $k$ queries and their accuracy responses.

We construct $\cA'$ from $\cA$ as follows. For each $x\in\cX$, let $\pi_x$ be an independent and uniformly sampled permutation over $[m]$. For $f\in\cF$, let $f_{\pi}\in\cF$ be $f_{\pi}(x):=\pi_x(f(x))$. Then, note that for any $f$, $f_{\pi}$ is distributed according to $\unim$.

Let $\cA^{\pi}$ be the following algorithm. If the first query of $\cA$ is $f^1$, then the first query of $\cA^{\pi}$ is $\pi(f^1)$. For $ i = 2,\ldots, k$, based on the previous queries $\{f^j_{\pi}, \forall j < i\}$ and outputs $\{\acc(f^j_{\pi}, \bz) \forall j < i\}$, if $\cA$ queries $f^i$, then $\cA^{\pi}$ queries $f^i_{\pi}$. Finally, if $\cA$ outputs $\hat{f}$, then $\cA^\pi$ outputs  $\pi(\hat{f})$. 
Now for any  $f^j$ and the true $f$,
\[
\acc(f^i_{\pi}, f) = \frac{1}{n} \sum^n_{i=1} \II_{\pi_{x_i}(f(x_i)) = f(x_i)}
= \frac{1}{n} \sum^n_{i=1} \II_{f(x_i) = \pi^{-1}_{x_i}(f(x_i))}
= \acc({f, \pi^{_1}(f)}).
\]
Similarly, it can be shown that for the final output
\begin{equation}
    \label{eq:red_final_features}
\acc(\cA^{\pi}, S_\cX, f) = \acc(\pi(\hat{f}), f) = \acc(\hat{f}, f_{\pi^{-1}})  =  \acc(\cA, S_{\cX},  f_{\pi^{-1}}).
\end{equation}
Therefore $\cA^{\pi}$ achieves the same expected accuracy on $f$ that $\cA$ achieves on $f_{\pi^{-1}}$. Alternatively, $\cA^{\pi}$ can be viewed as follows. If the first query of $\cA$ is $f^1$, $\cA^{\pi}$ queries $f^1$ on $f_{\pi^{-1}}$. Then for each $ i = 2,\ldots, k$, based on the previous queries $\{f^j, \forall j < i\}$ and accuracy responses $\{\acc(f^i, f_{\pi^{-1}}), \forall j < i\}$, if $\cA$ queries $f^i$, then $\cA^{\pi}$, queries $f^i$ on $f_{\pi^{-1}}$. Finally if $\cA$ returns output $\hat{f}$, then $\cA^{\pi}$ outputs $\hf_{\pi^{-1}}$ as an estimate of $f_{\pi^{-1}}$. Thus by~\eqref{eq:red_final_features},
\[
\expectationC{\cA, \pi}{\match{\cA^{\pi}, f}} = \expectationC{\cA}{\expectationC{\pi}{\acc(\cA^{\pi}, f)}} = \expectationC{\cA}{\expectationC{\pi}{\acc(\cA, f_{\pi^{-1}})}} 
= \expectationC{\cA}{\match{\cA, \unim}},
\]
where the last equality uses the fact that $f_{\pi^{-1}}$ is distributed according to $\unim$.
Hence,
\[
\match{\cA^{\pi}) = \min_{f}\expectationC{\cA, \pi}{\acc(\cA^{\pi}, f)}} = \expectationC{\cA}{\acc(\cA, \unim)}.
\]
Choosing $\cA'$ to be $\cA^{\pi}$, where $\pi$ are randomly chosen permutations proves the theorem.

\subsection{Algorithms}
\label{app:small}
We provide the complete algorithm for $\cAsmallunk$ and $\cAlargeunk$ in Figures~\ref{fig:smallk-unk} and~\ref{fig:largek-unk} respectively. As discussed in Section~\ref{sec:large-unk}, the proof for large values of $k$ is similar to that of Theorem~\ref{thm:large}. We now outline the sketch the proof for small values of $k$.
\begin{figure}[t]
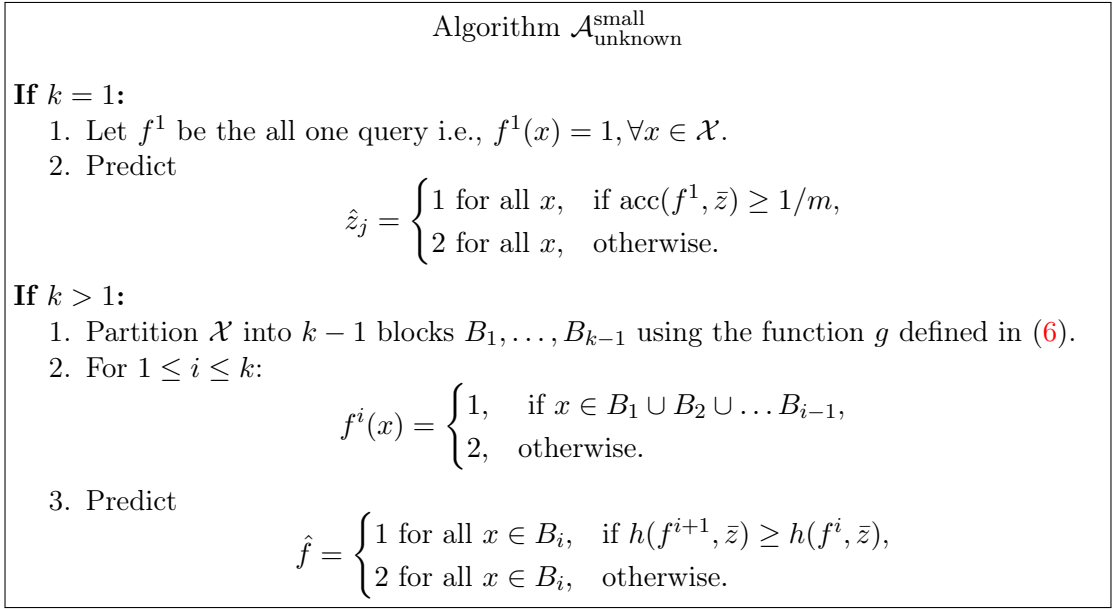

\begin{center}
\fbox{\begin{minipage}{0.95\textwidth}
\begin{center}
   Algorithm $\cAsmallunk$
\end{center}
\textbf{If $k = 1$:}
\begin{enumerate}
    \item Let $f^1$ be the all one query i.e., $f^1(x) = 1, \forall x\in\cX$. 
    \item Predict 
    \[
    \hat{z}_j=
    \begin{cases}
    1 \text{ for all $x$},& \text{if $\acc({f^1, \bz}) \geq 1/m$,}\\
    2 \text{ for all $x$},& \text{otherwise}.
    \end{cases}
    \]
\end{enumerate}
\textbf{If $k > 1$:}
\begin{enumerate}
\item 
Partition $\cX$ into $k-1$ blocks $B_1, \ldots, B_{k-1}$ using the function $g$ defined in~\eqref{eqn:divideX}.
\item
For $1\le i \leq k$:
    \[
    f^i(x)=
    \begin{cases}
    1, & \text{ if $x \in B_1 \cup B_2 \cup \ldots B_{i-1}$,}\\
    2, & \text{otherwise}.
    \end{cases}
    \]
\item Predict
\[
    \hat f=
    \begin{cases}
    1 \text{ for all $x \in B_i$},& \text{if $\match{f^{i+1}, \bz} \geq \match{f^{i}, \bz}$,}\\
    2 \text{ for all $x \in B_i$},& \text{otherwise}.
    \end{cases}
    \]
\end{enumerate}
\end{minipage}}
\end{center}
\caption{Algorithm for small values of $k$ without the test features.}
\label{fig:smallk-unk}
\end{figure}

\begin{figure}[t]
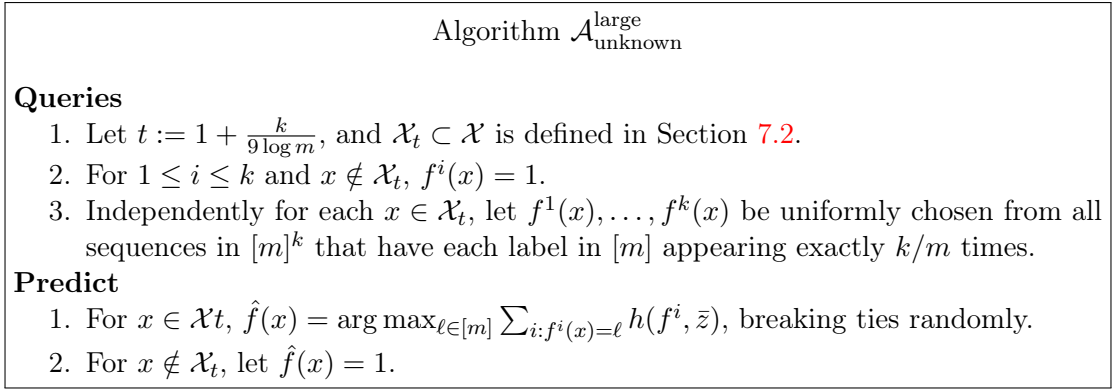

\begin{center}
\fbox{\begin{minipage}{0.95\textwidth}
\begin{center}
   Algorithm $\cAlargeunk$
\end{center}
\textbf{Queries}
\begin{enumerate}
\item Let $t := 1 + \frac{k }{9 \log m}$, and $\cX_t\subset \cX$ is defined in Section~\ref{sec:large-unk}.
\item For  $1\le i \leq k$ and $x\notin\cX_t$, $f^i(x) = 1$.
\item Independently for each $x\in\cX_t$, let $f^1(x), \ldots, f^k(x)$ be uniformly chosen from all sequences in $[m]^k$ that have each label in $[m]$ appearing exactly $k/m$ times.
\end{enumerate}
\textbf{Predict}
\begin{enumerate}
\item For $x\in\cX t$, $\hat{f}(x)=\arg\max_{\ell \in [m]}\sum_{i:f^i(x)=\ell}\match{f^i, \bz}$, breaking ties randomly.
\item For $x\notin\cX_t $, let $\hat{f}(x) = 1$.
\end{enumerate}
\end{minipage}}
\end{center}
\caption{Algorithm for large values of $k$ without the test features.}
\label{fig:largek-unk}
\end{figure}
The analysis of $\cAsmallunk$  for $k > 1$ is similar to that of Theorem~\ref{thm:small}, except we need to incorporate the condition that each $B_i$ is now not guaranteed to have size $n/(2(k-1))$. We modify the proof of Theorem~\ref{thm:small} as follows.

Recall that Let $N_{i,\ell}$ be the number of examples in $B_i$ with label `$\ell$'. Then $(N_{1,\ell}, \ldots, N_{m, \ell})$ is Multinomial $(|B_i|;\frac1m, \ldots, \frac 1m)$. Our final predictions correctly predicts $\max\{N_{i, 1}, N_{i,2}\}$ examples in $B_i$. Hence, summing over the blocks the total expected number of correct predictions
by our algorithm conditioned on $B_i$ is
\begin{align*}
&\quad\sum_{i=1}^{k-1}\expectation{\max\{N_{i,1}, N_{i,2}\}} \\
& = \sum_{i=1}^{k-1}\expectation{\max\{N_{i,1}, N_{i,2}\}}1_{|B_i|\geq n/(2(k-1))}
+ \sum_{i=1}^{k-1}\expectation{\max\{N_{i,1}, N_{i,2}\}}1_{|B_i|< n/(2(k-1))} \\
& \geq \sum_{i=1}^{k-1}\expectation{\max\{N_{i,1}, N_{i,2}\}}1_{|B_i| \geq n/(2(k-1))} +   \sum_{i=1}^{k-1}\expectation{\frac{N_{i,1}+ N_{i,2}}{2}}1_{|B_i|< n/(2(k-1))}  \\
& \geq \sum_{i=1}^{k-1} \frac{|B_i|}{m} +  \frac{(k-1)}{4}\sqrt{\frac{n}{2(k-1)m}} 1_{|B_i| \geq n/(2(k-1))} \\
& = \frac{n}{m} + \sum_{i=1}^{k-1} \frac{(k-1)}{4}\sqrt{\frac{n}{2(k-1)m}} 1_{|B_i| \geq n/(2(k-1))},
\end{align*}
where the second inequality follows from Lemma~\ref{lem:bin-max}.
Recall that for a binomial distribution, the median is larger than $\lfloor n p \rfloor$.  
The lemma follows by observing that since $|B_i| \sim \bin{n}{\frac{1}{k-1}}$ and $n/((k-1)) \geq 2$, $\Pr(|B_i| \geq n/(2(k-1)) \geq 1/2)$.


\section{Proof of Theorem~\ref{thm:info}}
\label{app:info}
By Corollary~\ref{cor:minmax}, 
\[
\max_{\cA} \match{\cA} = \max_{\cA} \match{\cA, \uni}.
\]
Hence it suffices to consider sequences generated by the uniform distribution. 
We first argue that there is a deterministic algorithm that maximizes $\match{\cA, \uni}$.
Let $\sA$ be the set of all deterministic algorithms. Let $\cA^*$ be the optimal algorithm. Recall that any randomized algorithm can be written as a distribution over deterministic algorithms.  Let $\lambda_{\cA}$ is the probability that the randomized algorithm $\cA^*$ assigns to a deterministic algorithm $\cA \in \sA$. Then,
\[
\match{\cA^*, \uni} = \sum_{\cA \in \sA} \lambda_{\cA} \match{\cA, \uni}
\leq \max_{\cA \in \sA}  \match{\cA, \uni},
\]
Hence, there exists a deterministic algorithm which performs as good as $\cA^*$ and there exists an optimal deterministic algorithm.

Since the algorithm is deterministic, by the symmetry of $\uni$, it suffices to consider the first query $\bq$ as the all one sequence i.e., $q_i = 1$ for $i=1, \ldots, n$. After this query, let $\hbz$ be the estimate of the optimal deterministic algorithm.
\begin{align*}
\match{\cA, \uni} &= \EE_{z \in \uni}[\match{\bz, \hbz}] \\
& =  \EE_{\match{\bq,\bz}}[\EE_{z \sim \uni}[\match{\bz, \hbz} | \match{\bq, \bz}]] \\
 & = \sum^n_{i=1} \EE_{\match{\bq,\bz}}
 [\EE_{z \sim \uni}[\match{z_i, \hz_i} | \match{\bq, \bz}]],
\end{align*}
where the first equality follows by law of conditional expectations and the second equality follows by the linearity of expectations. Without loss of generality consider $i=1$,
\begin{align*}
\EE_{z \in \uni}[\match{z_1, \hz_1} | \match{\bq, \bz} = r]
& = \sum_{j \in [m]} \Pr(z_1 = i | \match{\bq, \bz} = r) \match{z_1, \hz_1} \\
& \leq \max_{j\in [m]} \Pr(z_1 = i | \match{\bq, \bz} = r) \\
& = \max_{j\in [m]}  \Pr(z_1 = i | \match{\bq, \bz} = r) \\
 & = \max_{j\in [m]} \frac{\Pr(\match{\bq, \bz} = r | z_1 = i) \Pr(Z_1 = i)}{\Pr(\match{\bq, \bz} = r)},
\end{align*}
where the last equality follows by Bayes rule. Note that $\match{\bq, \bz}  \sim \bin{n}{1/m}$ and
conditioned on $z_1=1$ $\match{\bq, \bz} \sim \bin{n-1}{1/m}+ 1$ and conditioned on $z_1 = j \neq 1$,
$\match{\bq, \bz}  \sim \bin{n-1}{1/m}$. Hence, the above quantity can be simplified to
\begin{align*}
 \max_{j\in [m]} \frac{\Pr(\match{\bq, \bz} = r | z_1 = i) \Pr(Z_1 = i)}{\Pr(\match{\bq, \bz} = r)} 
  & = \frac{1}{m} \max \left( \frac{{n-1 \choose r-1}m}{ {n \choose r} } , \frac{{n-1 \choose r}m}{{n \choose r}(m-1)} \right) \\
  & = \frac{1}{m} \max \left( \frac{rm}{n},\frac{(n-r)m}{n(m-1)} \right) \\
  & = \frac{1}{2m} \left( \frac{rm}{n} + \frac{(n-r)m}{n(m-1)} 
 + \left| \frac{rm}{n}- \frac{(n-r)m}{n(m-1)}  \right| \right) \\
 & =  \frac{1}{2m} \left( \frac{rm}{n} + \frac{(n-r)m}{n(m-1)} + \frac{m|rm-n|}{n(m-1)} \right).
\end{align*}
Let $ r= \match{\bq, \bz}$. Then, $r \sim \bin{n-1}{1/m}$ Combining the above equations together, 
\begin{align*}
\match{\cA^*, \uni} 
& \leq \frac{1}{2m} \EE_{r} \left[  \frac{rm}{n} + \frac{(n-r)m}{n(m-1)} + \frac{m|rm-n|}{n(m-1)} \right] \\
& = \frac{1}{m} + \frac{ \EE[|rm-n|] }{2n(m-1)} \\
& \leq \frac{1}{m} + \frac{ \sqrt{\EE[(rm-n)^2]}}{2n(m-1)} \\
& = \frac{1}{m} + \frac{1}{2\sqrt{n(m-1)}},
\end{align*}
where we used $\EE\left[\left(r-\frac{n}{m} \right)^2\right] = n\cdot\frac 1m \cdot (1-\frac 1m)$.

\end{document}